\documentclass[11pt]{article}

\usepackage{times}
\usepackage{amssymb}
\usepackage{amsmath}
\usepackage{amsthm}
\usepackage{latexsym}
\usepackage{graphicx}
\usepackage{color}
\usepackage{graphics}
\usepackage[round]{natbib}
\usepackage{cancel}
\usepackage{mathtools}
\usepackage{thm-restate}
\usepackage[mathscr]{euscript}
\usepackage{hyperref}
\usepackage[margin=1.25in]{geometry}
\usepackage[toc,page,header]{appendix}
\usepackage{booktabs}
\usepackage{longtable}

\usepackage{MnSymbol}
\DeclareMathAlphabet\mathbb{U}{msb}{m}{n}
\usepackage{xpatch}

\def\Rset{\mathbb{R}}

\DeclareMathOperator*{\E}{\mathbb E}
\DeclareMathOperator*{\argmax}{argmax}

\DeclareMathOperator{\Ind}{\mathbb{I}} 

\DeclarePairedDelimiter{\abs}{\lvert}{\rvert} 
\DeclarePairedDelimiter{\bracket}{[}{]}
\DeclarePairedDelimiter{\curl}{\{}{\}}

\DeclarePairedDelimiter{\paren}{(}{)}

\declaretheorem{theorem}
\newtheorem{lemma}[theorem]{Lemma}

\newcommand{\sC}{{\mathscr C}}
\newcommand{\sD}{{\mathscr D}}
\newcommand{\sF}{{\mathscr F}}

\newcommand{\sH}{{\mathscr H}}

\newcommand{\sM}{{\mathscr M}}

\newcommand{\sR}{{\mathscr R}}

\newcommand{\sX}{{\mathscr X}}
\newcommand{\sY}{{\mathscr Y}}

\newcommand{\Rall}{\sR_{\mathrm{all}}}

\newcommand{\Ic}{\bar I_c}

\newcommand{\ov}{\overline}

\newcommand{\e}{\epsilon}
\newcommand{\set}[2][]{#1 \{ #2 #1 \} }
\newcommand{\ignore}[1]{}
\newcommand{\softmax}{\textrm{softmax}}
\newcommand{\logits}{\textrm{logits}}
\newcommand\examplebox[1]{\fbox{\begin{minipage}{.46\textwidth}\footnotesize{#1}\end{minipage}}}

\newcommand\ec[1]{{[\color{cyan}{\textbf{Eunsol:}~#1}]}}
\usepackage[algo2e,ruled]{algorithm2e}

\hypersetup{
  colorlinks   = true,
  urlcolor     = blue,
  linkcolor    = blue,
  citecolor    = blue
}

\title{Learning to Reject with a Fixed Predictor:\\ Application to Decontextualization}
\author{Christopher Mohri, Daniel Andor, Eunsol Choi, Michael Collins}
\date{}

\begin{document}
\maketitle 

\begin{abstract}
We study the problem of classification with a reject option for a fixed predictor, applicable in natural language processing. \ignore{where many correct labels are often possible} We introduce a new problem formulation for this scenario, and an algorithm minimizing a new surrogate loss function. We provide a complete theoretical analysis of the surrogate loss function with a strong $H$-consistency guarantee. For evaluation, we choose the \textit{decontextualization} task, and provide a manually-labelled dataset of $2\mathord,000$ examples. Our algorithm significantly outperforms the baselines considered, with a $\sim\!\!25\%$ improvement in coverage when halving the error rate, which is only $\sim\!\! 3 \%$ away from the theoretical limit. 
\end{abstract}

\section{Introduction}

Large language models, often trained with billions of parameters, have achieved impressive performance in recent years \citep{Raffel2019} and are used in a wide variety of natural language generation tasks. 
However, their output is sometimes undesirable, with \emph{hallucinated content} \citep{maynez2020, Filippova2020}, and much work remains to fully understand their properties\ignore{and develop theoretical learning guarantees for their performance, }.

In many applications, such as healthcare,  question-answering systems, or customer service, incorrect predictions are particularly costly and must be avoided. This motivates the design of algorithms for large language models and other NLP tasks that achieve high precision on a large fraction of the input set, while abstaining on the rest.
How can we devise such accurate models that allow a reject option? 

A common technique adopted in the past is that of \emph{confidence-based models}, where rejection is defined by some threshold on the predictor's scores and admits a fixed cost \citep{Hendrickx2021}. 
\citet{Chow1957, Chow1970} was the first to provide an analysis of the trade-off between error and rejection rate, as well as the associated Bayes-optimal solution. 
The rejection rule was later studied based on the receiver operating characteristic (ROC) curve \citep{Tortorella2000, Santos-Pereira2005, Pietraszek2007, Landgrebe2006}. Some subsequent work has focused on minimizing surrogate loss functions for the cost-based objective, with various theoretical guarantees \citep{Bartlett2008, Grandvalet2008, Yuan2010}. In NLP, it has been reported that scores from popular large language models such as T5, BART, and GPT-2 are poorly calibrated \citep{Jiang2020, Kumar2019, BART, Raffel2019}, similar to modern neural networks \citep{Chuan2017}, but \citet{xin-etal-2021-art} proposed a simple regularization trick during training to improve these scores. 
The method used by several other
authors can also be viewed as an instance of confidence-based models \citep{KamathJiaLiang2020, Choi2021, garg-2021,Dong2018, Varshney2022TowardsIS}. Here, the idea consists of first learning a scoring function defined over pairs $(x, y)$ from a function mapping $\sX$ to $\sY$, and next applying the confidence-based technique to the scoring function.

However, as shown by \citet{CortesDeSalvoMohri2016}[see Figure 2], straightforward confidence-based methods are in general suboptimal.
When the predictor learned is not the Bayes optimal solution, 
in general, a more complex rejection rule is needed to achieve a smaller loss. The authors suggested instead seeking a suitable \emph{rejector} out of a family of rejection functions that may be richer than that of confidence-based threshold ones. They gave theory and algorithms for learning the predictor and rejector simultaneously by minimizing a \emph{rejection loss}, whose definition is based on the cost of rejection, $c$. More recently, an extension of this work to the multi-class setting was given by \citet{Sugiyama2021}, based on a specific family of rejection
function, thereby resolving a question raised in \citep{Ni2019}.

We aim to design accurate models with a rejection
option for NLP generation tasks such as \textit{decontextualization} \citep{Collins2021}.
Decontextualization involves editing a sentence within a passage such that it can stand alone. It is critical in this task to make confident predictions, or abstain; if we are to edit an author's words, we must be sure to do so correctly. 

One solution
would consist of adopting the rejection loss function of \citet{CortesDeSalvoMohri2016}
(or \citet{Sugiyama2021}) and of seeking a model that minimizes
the rejection loss to learn, simultaneously, a predictor $f$ and
a rejector $r$. However, for NLP tasks such as decontextualization,
this faces a key obstacle: the full set of \emph{accurate outputs} for a given
input sentence may be
large and is typically not at the learner's disposal. For example, in decontextualization, 
some standard transformations such as passivation applied to an accurate output
can immediately lead to a large number of accurate outputs.
To minimize the rejection loss, however, the learner must be able to
evaluate the correctness of any potential predictor on any example in
the training sample. But, apart from the single label in the training
sample, other potential accurate labels are not provided and thus the
correctness of a potential predictor cannot be checked.  How can we
then learn an accurate rejection-based model for such settings?

One way to proceed is instead to first learn a predictor $f$, using
the standard techniques adopted for large language models, for example
by minimizing the cross-entropy loss in next-token prediction. Next, given
that predictor,
it is not hard to manually assign a binary label to the output of f on some relatively small set of held-out examples indicating their correctness.
The problem then consists of minimizing the rejection loss in which $f$ is fixed 
and where we seek to find the best rejector function $r$, where the rejector function $r$ takes the input $x$ together with the output $f(x)$ as its own input.
This is what we will refer to as \emph{learning to reject with a
fixed predictor} $f$.

The resulting binary rejection loss function for $r$ is hard to
optimize and, instead, we need to resort to a surrogate loss. \citet{CortesDeSalvoMohri2016}
gave a consistent surrogate loss for their full rejection loss,
which involved learning both $f$ and $r$. However, since $f$ is fixed in
our context, we cannot benefit from the corresponding consistency
guarantees. Instead, we first define a parametric family of surrogate losses inspired by their
work for learning $r$ alone. Next, we prove that these surrogate losses benefit from strong consistency
results, when a suitable constraint holds for the parameters.

Our results make use of the recent work of \citet{AwasthiMaoMohriZhong2022},
which gave general tools for \emph{$H$-consistency bounds}. These are
bounds relating directly the excess error or estimation error of
the surrogate loss to those of the original loss (here the
rejection loss). Thus, they are stronger guarantees than
asymptotic Bayes consistency results. Furthermore, they can
be extended to other hypothesis sets $H$ than that of the family of
all measurable functions. We use those tools to prove the first $H$-consistency bound for our
surrogate rejection loss, which provides a strong justification for
its adoption for tackling our problem.

The rest of this paper is organized as follows. In Section~\ref{sec:learning-problem}, we
formulate our learning problem, which consists of learning a rejector
given a fixed predictor. In Section~\ref{sec:learning-methods}, we briefly describe
confidence-based models and contrast them with our two-step model.  Section~\ref{sec:surr-loss} introduces
our surrogate rejection loss. In Section~\ref{sec:consistency-bound}, we prove a strong
$H$-consistency bound for our surrogate rejection loss. In Section~\ref{sec:decontext-task},
we give a brief description of the decontextualization task and annotation procedure. Section~\ref{sec:experiments} 
reports the results of our experiments with our surrogate
rejection loss algorithm, and shows that it compares favorably with several baselines in a decontextualization task.

\ignore{
\textbf{Contributions}.
We present a theoretically-justified approach to classification with a reject option for applications where the predictor remains fixed. Our main contributions include the following: 
(1)  a new formulation of the problem of learning a rejector with fixed predictor (for cases where there may be many correct labels);
(2) introduction of a new surrogate loss function for our scenario, with the proof of a strong $H$-consistency bound guarantee;
(3)  definition of the notion of correctness for decontextualization, and its use to provide a dataset of $2\mathord,000$ manually-labeled decontextualizations;
(4) experimental results demonstrating a $10$-$25\%$ improvement over baselines in coverage at various precision levels on decontextualization.
}

\section{Learning problem}
\label{sec:learning-problem}

We consider the problem of  
\emph{sequence-to-sequence modeling with high confidence} in natural language processing. The general
objective is to design an algorithm that only returns an output when
it is highly likely to be correct, while still
guaranteeing a high coverage. 

Let $\sX$ denote the input and $\sY$ the output set of sequences and
let $\sD$ be a distribution over $\sX \times \sY$.  The problem can be
formalized in terms of a sequence-to-sequence \emph{predictor} $f
\colon \sX \to \sY$ and a \emph{rejector} $r\colon \sX \times \sY \to
\Rset$. A non-positive sign for the output of the rejector is
interpreted as rejection, and a positive one as acceptance. Formally,
on an input $x \in \sX$, we have
\begin{align*}
    (f, r)(x) =  \begin{cases} 
      f(x), & \text{if } r(x, f(x)) > 0 \\
      \text{reject,} & \text{if } r(x, f(x)) \leq 0 .
   \end{cases}
\end{align*}
Given a hypothesis set $\sF$ of sequence-to-sequence predictors and a
family of rejectors $\sR$, for a coverage $\gamma \geq 0$, the
problem can be formulated as follows:
\begin{align}
\label{eq:coverage}
  \min_{f \in \sF, r \in \sR} \ \E_{(x, y) \sim \sD} & \bracket*{\Ind_{f(x) \neq y} \, \Ind_{r(x, f(x)) > 0}}\\ \nonumber
  s.t.\ \qquad \E & \bracket*{\Ind_{r(x, f(x)) > 0}} \geq \gamma.
\end{align}
We wish to minimize the error $\paren*{\Ind_{f(x) \neq y}}$ on
accepted examples $\paren*{\Ind_{r(x, f(x)) > 0}}$, while ensuring that at least $\gamma$
fraction of the examples are accepted. We will refer to the accuracy on accepted examples
as the \textit{precision}.
This problem and its formulation in terms of a predictor and a
rejector are closely related to the rejection or abstention setting
introduced by
\citet{CortesDeSalvoMohri2016,CortesDeSalvoMohri2016bis}. In fact, for
any coverage parameter $\gamma$ there exists a corresponding rejection
cost $c \in (0, 1)$, so that the precision-coverage problem can be equivalently
formulated as the following optimization problem \citep{CortesDeSalvoMohri2016}:
\begin{align}
\label{eq:rejection}
  \min_{f \in \sF, r \in \sR} & \ \E_{(x, y) \sim \sD} \bracket*{L(f, r, x, y)},
\end{align}
where the loss function $L$ is defined as follows:
\ifdim\columnwidth=\textwidth
{
\begin{equation}
\label{eq:l_binary}
  L(f, r, x, y)
  = \Ind_{f(x) \neq y} \, \Ind_{r(x, f(x)) > 0} + c \Ind_{r(x, f(x)) \leq 0}.
\end{equation}}
\else 
{
\begin{multline}
\label{eq:l_binary}
  L(f, r, x, y) \\ 
  = \Ind_{f(x) \neq y} \, \Ind_{r(x, f(x)) > 0} + c \Ind_{r(x, f(x)) \leq 0}.
\end{multline}
}\fi
As the cost of rejection $c$ increases, one can expect a higher coverage, but lower precision.
One difference with respect to the framework of
\citet{CortesDeSalvoMohri2016} is that, here, the rejector takes as
argument the prediction as well. More generally, as already
pointed out in that previous work, the rejection cost $c$ can be
a function of $x$ and $f(x)$, not just a constant.
The problem of selective classification \citep{GelbhartElYaniv2019}, which is based on the choice
of a threshold function, can be viewed as a special case of this
framework.

\textbf{Distribution model}. Before describing any method for tackling this problem, we wish
to discuss the distributional model adopted
in certain NLP tasks such as decontextualization. In standard learning tasks, there
may be multiple correct labels $y$ for the same input $x$ and with an i.i.d.\ training sample, 
we expect to come across all of these labels with
their conditional probabilities given $x$. 

In complex NLP tasks such as decontextualization, however, there may be a relatively large number of
correct $y$s for a given $x$: the task is highly 
non-deterministic. As discussed, 
some standard transformations such as passivation
applied to one $y$ can immediately lead to a much
larger number of correct output sentences. Thus, 
it is not realistic to demand from labelers to supply all possible correct sentences $y$ for a given input $x$. On the other hand,
we do not wish to consider it to be an error if
a model returns a desirable sentence that does not exactly 
match any of the labels provided by the labelers. 

What should be the correct distributional model
to adopt for such NLP tasks? We can consider two
models: a \emph{deterministic model} where only
a single label is accepted as correct for any
input sentence; or a more general and more useful \emph{stochastic} or \emph{non-deterministic
model} where, in addition to the single label $y$ 
(or few labels) provided by labelers, any other
label $y'$ returned by the model is viewed as correct
provided that $y'$ is sufficiently similar to
$y$, based on some pre-defined similarity measure. It is important to note that
this pre-defined similarity measure may be difficult to specify, and may therefore require expert human annotation. 

Adopting the non-deterministic model, for $(x, y) \sim \sD$ and a given $f \in \sF$,
we amend our indicator function $\Ind_{f(x) \neq y}$ measuring incorrectness in \eqref{eq:coverage}. 
Instead, we measure $\Ind_{f(x)\notin A_{x,y}}$, where $A_{x,y} \subseteq \sY$ 
implements some similarity measure and describes a set of acceptable outputs $y$.
Provided that we have a boolean random variable $a \in \{-1, +1\}$ derived from $(x, y)$ and $f(x)$ indicating membership to $A_{x,y}$, the event where $f(x)$ 
is an acceptable output,
we can simplify this indicator function to  $\Ind_{a \neq -1}$. Since $f$ is fixed, we remove it as an 
argument to $r$, and the distribution over $(x, y) \sim \sD$ leads to a distribution $(x, a) \sim \ov \sD$ induced by $f$.
We will refer to the following as the \textit{induced rejection loss} defined for any rejector $r$ and pair $(x, a)$:
\begin{equation*}
\label{eq:nondeterministic_rejectionloss}
  \ell(r, x, a) = \Ind_{a \neq -1} \, \Ind_{r(x) > 0} + c \Ind_{r(x) \leq 0}.
\end{equation*}
In the following, we will distinguish two methods for learning
the rejection function $r$: the so-called \emph{confidence-based method}
where $r$ is simply defined as a threshold based on the predictor
$f$, and a \emph{two-step learning method} where first $f$ is learned and
then $r$.

\section{Learning methods}
\label{sec:learning-methods}

In this section, we describe in more detail the two learning methods previously mentioned.

\subsection{Confidence-based method}

The confidence-based method is perhaps the most commonly used
one to define a rejection function. Let $f(x) = \argmax_y s(x, y)$ for some scoring function $s : \sX \times \sY \mapsto \sR$, which could be $p(y | x; \omega)$, where $\omega$ represents model parameters. Then, a threshold value
$\theta$ is set based on some function of the scores $s(x, y)$ assigned to each $y \in \sY$ for a given $x \in \sX$. If $s_i$ and $s_j$ are the final
scores assigned to $(x_i, y_i) \sim \sD$ and $(x_j, y_j) \sim \sD$ respectively, we wish to have \textit{monotonicity}: $s_i \geq s_j \Leftrightarrow \Ind_{f(x_i) \neq y_i} \leq \Ind_{f(x_j) \neq y_j}$ \citep{Yaniv2017}.

The \emph{MaxProb} method is simply defined in
terms of the highest score \citep{Maxprob}. In that case, the rejection
function is a function mapping from $\sX$ to $\Rset$
defined by
\[
r(x) = s(x, f(x)) - \theta,
\]
for some choice of the threshold $\theta \in \Rset$. Another popular scoring function is based on Monte-Carlo dropout \citep{Smith2018MCdropout, Zoubin2016}, measuring statistics such as the mean or negative variance of the scores $s(x, y)$.

Fitting the threshold to guarantee that it matches a target precision has been studied \citep{Yaniv2017}. In our experiments, we follow a simpler procedure, as this is not our focus; we are interested in the quality of the underlying scores.

\subsection{Two-step method}

In the two-step method, a predictor $f$ is first learned by minimizing
a loss function for $\Ind_{f(x) \neq y}$, such as the cross-entropy loss. Next, the rejection function $r$ is learned as a binary classifier.

Note that, to learn $r$ in the non-deterministic distributional model requires binary labels for
pairs $(x, f(x))$ indicating if the output sequence $f(x)$ is indeed
a good label for $x$, or formally, if $\Ind_{f(x) \in A_{x, y}}$. As already discussed,
that information cannot be directly derived from
the label $y$ of $x$ in the training sample since the correct labels
are typically not unique in NLP tasks.
One way to derive that information is to manually label 
such pairs. This admits two disadvantages: the manually
assigned labels are specific to the predictor $f$ previously
learned and thus cannot be reused for different predictors; and of course this requires manual labeling which is typically costly. However, if the cost is not too significant, one can label a moderately-sized dataset and then train a classifier on this dataset. 

One approach for training such a classifier is to use the standard cross-entropy loss function. More specifically, for pairs $((x, f(x)), a)$, where $a$ represents the label or annotation, one can train with direct supervision using the binary cross-entropy loss. However, similar to the \textit{MaxProb} method, this also requires setting a threshold $\theta$ for acceptance based on model scores. One can only hope that the classifier produces higher-quality scores, where quality is associated with monotonicity. Thus, both methods are based on straightforward threshold rejection. Additionally, to the best of our knowledge, minimizing the cross-entropy loss does not have any proven guarantee with respect to our main objective: minimizing the induced rejection loss. In the next section, we tackle both of these problems: we introduce a new loss function with a built-in threshold that \textit{directly minimizes} the induced rejection loss. 

\section{Surrogate rejection loss}
\label{sec:surr-loss}

\citet{CortesDeSalvoMohri2016} study the joint optimization problem
in~\ref{eq:rejection} where the predictor $f$ is a binary classifier, and define a convex surrogate loss upper-bounding their rejection loss $L$. We use the same technique to upper-bound the induced rejection loss. Specifically, the following inequality holds, where $\alpha$ and $\beta$ are positive parameters and $x \mapsto \Phi(-x)$ and $x \mapsto \Psi(-x)$ are convex functions upper-bounding $\Ind_{u \leq 0}$:
\begingroup 
\addtolength\jot{3pt} 
\begin{align*}
\label{eq:binaryrejection}
    \ell(r, x, a)  
    & = \Ind_{a \neq -1} \Ind_{r(x) > 0} + c\Ind_{r(x) \leq 0} \\
    & = \Ind_{a \leq 0} \Ind_{r(x) > 0} + c\Ind_{r(x) \leq 0} \\
    & \leq \max \curl*{\Ind_{a \leq 0} \Ind_{-r(x) < 0}, c\Ind_{r(x) \leq 0}} \\
    & \leq \max \curl*{\Ind_{\max (a, -r(x)) \leq 0}, c\Ind_{r(x) \leq 0}}.
\end{align*}
Next, since the maximum is
lower-bounded by the average, we can write:
\begin{align*}
    \ell(r, x, a)
    & \leq \max \curl*{\Ind_{ \frac{a -r(x)}{2}} \leq 0, c\Ind_{r(x) \leq 0}} \\
    & = \max \curl*{\Ind_{ \alpha \frac{a -r(x)}{2} \leq 0}, c\Ind_{\beta r(x) \leq 0}} \\
    & \leq \max \curl*{\phi \paren*{\tfrac{\alpha}{2} (r(x) - a}, c \psi \paren*{- \beta r(x)}} \\
    & \leq \phi \paren*{\tfrac{\alpha}{2} (r(x) - a} + c \psi \paren*{- \beta r(x)}.
\end{align*}
\ignore{
The sign of $yf(x)$ can be used to determine correctness in a binary classification setting, thereby replacing $f(x) \neq y$, and its magnitude can be interpreted as a measure of confidence. As discussed, measuring correctness in our case may require an expert human annotation. If we fix $f$ and only update $r$, we can manually annotate outputs of $f$ as correct ($+1$) or incorrect ($-1$). Note that doing so omits any sense of magnitude from $yf(x)$. In our proposed surogate loss function, we replace $yf(x)$ with the new
annotated label $a \in \set{-1, +1}$ and use the exponential function for $\phi$ and $\psi$, to derive a similar convex surrogate loss, this time directly upper-bounding the induced rejection loss:

\ifdim\columnwidth=\textwidth
\begin{align*}
    \ell(r, x, a) 
    = \Ind_{a \neq -1} \Ind_{r(x) > 0} + c\Ind_{r(x) \leq 0}
    & = \Ind_{a \leq 0} \Ind_{r(x) > 0} + c\Ind_{r(x) \leq 0} \\
    & \leq \phi \paren*{\frac{\alpha}{2} (r(x) - a)} + c \psi \paren[\Big]{- \beta r(x)} \\
    & = e^{\frac{\alpha}{2}[r(x) - a]} + c e^{-\beta r(x)}.
\end{align*}
\else
\begin{align*}
    & \ell(r, x, a) = \Ind_{a \neq -1} \Ind_{r(x) > 0} + c\Ind_{r(x) \leq 0} \\
    & = \Ind_{a \leq 0} \Ind_{r(x) > 0} + c\Ind_{r(x) \leq 0} \\
    & \leq \phi \paren*{\frac{\alpha}{2} (r(x) - a)} + c \psi \paren[\Big]{- \beta r(x)} \\
    & = e^{\frac{\alpha}{2}[r(x) - a]} + c e^{-\beta r(x)}.
\end{align*}
\fi
}
We use the exponential function for $\phi$ and $\psi$, giving our surrogate loss function:
\begin{equation*}
    \ell(r, x, a) \leq e^{\frac{\alpha}{2}[r(x) - a]} + c e^{-\beta r(x)}.
\end{equation*}
While the induced rejection loss $\ell(r, x, a)$ is provably NP-hard to optimize, our surrogate loss is convex and differentiable. A key insight is that models optimized with our loss function have a built-in threshold of 0; they are directly optimized for a \textit{specific} precision. Thus, there is no need to further specify some score threshold as in all methods previously described. In those methods, one can still target certain precision levels through the choice of threshold, but the \textit{underlying scores} are not necessarily favorable for that precision level. 

\ignore{
We observe that analyzing all coverage levels requires training and tuning hyper-parameters for several different values of the cost of rejection $c$. However, in practice, one is often only concerned with one coverage level, so this does not produce an added cost.  
}

While \citet{CortesDeSalvoMohri2016} prove theoretical guarantees for a joint minimization of their loss function in their binary setting, these naturally do not apply to our problem (the predictor or does not have zero error and is not the Bayes predictor). In the next section, we prove strong theoretical guarantees for minimizing our surrogate loss function. 

\section {$\sH$-consistency bound}
\label{sec:consistency-bound}

In this section we prove an $\sH$-consistency bound, a concept introduced by \citet{AwasthiMaoMohriZhong2022}, of our surrogate loss function
with respect to the rejection loss. To the best of our knowledge, these non-asymptotic bounds are the strongest guarantees known regarding the minimization of surrogate loss functions. We first introduce some basic concepts and adopt the notation of \citet*{AwasthiMaoMohriZhong2022}.

\subsection{Preliminaries}

Let $\sX$ denote the input space and $\sY = \curl*{-1,+1}$ the binary
label space.
We will denote by $\sD$ a distribution over $\sX \times \sY$. Let $\sR$ denote a family of rejection functions mapping from $\sX$ to
$\Rset$.
Then, the \emph{generalization error} $R_{\ell}(r)$ and
\emph{minimal generalization error} $R_{\ell, \sR}^*$ for a loss function $\ell(r, x, y)$
are defined by 
\[
R_{\ell}(r) = \E_{(x, y) \sim
  \sD}\bracket*{\ell(r, x, y)} \text{ and } R_{\ell, \sR}^* = \inf_{r \in
  \sR}R_{\ell}(r).
  \]
We will adopt the standard notation
for the conditional distribution of $Y = 1$ given $X = x$: $\eta(x) =
\sD(Y = 1 \!\mid\! X = x)$. The
generalization error can be expressed as
$R_{\ell}(r) = \E_{X}\bracket*{\sC_{\ell}(r,x)}$, where
$\sC_{\ell}(r,x)$ is the \emph{conditional $\ell$-risk} defined by
$\sC_{\ell}(r,x) = \eta(x)\ell(r, x, +1) + (1 - \eta(x))\ell(r, x,
-1)$.  The \emph{minimal conditional $\ell$-risk} is denoted by
$\sC_{\ell,\sR}^*(x) = \inf_{r\in \sR}\sC_{\ell}(r,x)$. We also use
the following shorthand for the gap $\Delta\sC_{\ell,\sR}(r,x) =
\sC_{\ell}(r,x) - \sC_{\ell,\sR}^*(x)$.

A key quantity that appears in their bounds is the \emph{$\paren*{\ell,
  \sR}$-minimizability gap} $\sM_{\ell,\sR}$, which is the difference
of the best-in class error and the expectation of the minimal
conditional $\ell$-risk: 
\[ \sM_{\ell,\sR}
 = R^*_{\ell,\sR} - \E_{X} \bracket* {\sC^*_{\ell,\sR}(x)}.
 \]  
 This is an inherent property of the hypothesis set $\sR$ and
 distribution $\sD$ that we cannot hope to estimate or minimize. As discussed later,
 the minimizability gap is zero when $\sR$ is the family of all measurable functions.

\subsection{Definition of the losses and the desired guarantee}

We consider the induced rejection loss function $\ell_2 = \ell$ defined for any rejection function or \emph{rejector} $r \colon \sX \to \Rset$ and $(x, a) \in \sX \times \set{-1, +1}$ by
\begin{align}
    \ell_2(r, x, a) = \Ind_{a=-1}\Ind_{r(x) > 0} + c \Ind_{r(x) \leq 0}.
\end{align}
We will consider a surrogate loss function $\ell_1$ parameterized by $\alpha, \beta > 0$ and defined for any rejection function or \emph{rejector} $r \colon \sX \to \Rset$ and $(x, a) \in \sX \times \set{-1, +1}$ by
\begin{align}
    \ell_1(r, x, a) =  e^{\frac{\alpha}{2}[r(x) - a]} + c e^{-\beta r(x)}.
\end{align}
We will prove $\sH$-consistency bounds for the surrogate loss $\ell_1$,
when $r$ is in the family of all measurable functions $\Rall$. These are excess error bounds of the form $R_{\ell_2}(r) - R^*_{\ell_2} \leq f(R_{\ell_1}(r) - R^*_{\ell_1})$ valid for all $r$ for an increasing function $f$. To do so, we will use the following general theorem from \citep*{AwasthiMaoMohriZhong2022}.

\begin{theorem}
\label{th:excess_bounds_Psi}
 Assume that there exists a convex function $\Psi\colon
 \Rset_+ \to \Rset$ with $\Psi(0) = 0$  such
 that the following holds for all $r\in \sR$ and $x\in \sX$:
$\Psi\paren*{\Delta\sC_{\ell_2,\sR}(r, x)} \leq \Delta\sC_{\ell_1,\sR}(r, x)$.
Then, the following inequality holds for any $r \in \sR$:
\ifdim\columnwidth=\textwidth
{\begin{equation*}
    \Psi\paren*{R_{\ell_2}(r) - \sR_{\ell_2, \sR}^* + \sM_{\ell_2, \sR}} \\
    \leq  R_{\ell_1}(r) - R_{\ell_1,\sR}^* + \sM_{\ell_1, \sR} .
\end{equation*}}
\else
{
\begin{multline*}
    \Psi\paren*{R_{\ell_2}(r) - \sR_{\ell_2, \sR}^* + \sM_{\ell_2, \sR}} \\
    \leq  R_{\ell_1}(r) - R_{\ell_1,\sR}^* + \sM_{\ell_1, \sR} .
\end{multline*}
}
\fi
\end{theorem}
As shown by
\citet{steinwart2007compare}[lemma~2.5],
since 
$\ell_1(r, x, a)$
and $\ell_2(r, x, a)$
can be expressed in terms
of $r(x)$ and $a$ alone,
both minimizability gaps
$\sM_{\ell_{0-1},\Rall}$ and
$\sM_{\Phi,\Rall}$ vanish ($\sM_{\ell_{0-1},\Rall} = \sM_{\Phi,\Rall} = 0$).
To make use of this theorem, in the next sections, we will derive the expression
of the calibration gaps 
$\Delta\sC_{\ell_2}$ and 
$\Delta\sC_{\ell_1}$. 

\subsection{Calibration gaps}

The following gives the expression of the calibration gaps.
The proof for both results is deferred to the appendix.

\begin{restatable}{lemma}{CalibrationGapRejectionLoss}
\label{lemma:1}
The Bayes solution $r^*$ for the rejection loss can be expressed for all $x \in \sX$ by $r^*(x) = \eta(x) - (1 - c)$. The calibration gap for the rejection loss is given for any $r \in \Rall$ and $x \in \sX$ by
\[
\Delta \sC_{\ell_2} (r, x)
= |\eta(x) - (1 - c)| \Ind_{r(x) r^*(x) \leq 0}.
\]
\end{restatable}

\begin{restatable}{lemma}{CalibrationGapSurrogateLoss}
\label{lemma:2}
Let $I_\eta(x)$ be defined by
$I_\eta(x) =  \eta(x) e^{-\frac{\alpha}{2}} + (1 - \eta(x))e^\frac{\alpha}{2}$ and define $\gamma$ by $\gamma = \frac{\alpha}{\alpha + 2\beta}$. Then, the calibration gap for the surrogate loss is given by
\ifdim\columnwidth=\textwidth
{
\begin{align*}
\Delta \sC_{\ell_1} (r, x)
= e^{\frac{\alpha}{2} r(x)} I_\eta(x) + c e^{-\beta r(x)} 
- \frac{1}{1 - \gamma} \paren*{ \frac{2 \beta c}{\alpha} }^{\gamma}  I_\eta(x)^{1 - \gamma} .
\end{align*}
}
\else
{
\begin{align*}
\Delta \sC_{\ell_1} (r, x)
& = e^{\frac{\alpha}{2} r(x)} I_\eta(x) + c e^{-\beta r(x)} \\
& - \frac{1}{1 - \gamma} \paren*{ \frac{2 \beta c}{\alpha} }^{\gamma}  I_\eta(x)^{1 - \gamma} .
\end{align*}
}
\fi
\end{restatable}

\subsection{Bound}

In this section, we present our main result. 
A key challenge in finding a function $\Psi$ relating the two calibration gaps is that 
$\Delta \sC_{\ell_1}$ depends on the value $r(x) \in \Rset$, while $\Delta \sC_{\ell_2}$ only 
depends on the sign of $r(x)$, via that of $r^*(x)$.
The following provides a key solution to this 
problem.
\begin{figure}[t]
\centering
\includegraphics[scale=.35]{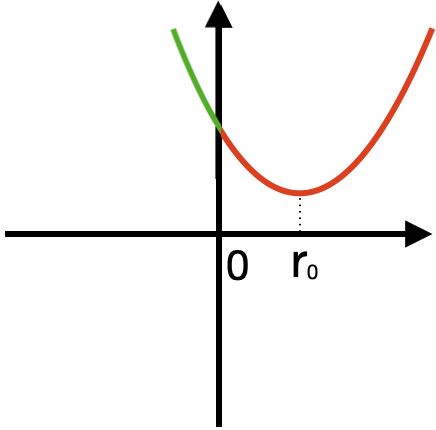}
\hspace{1cm}
\includegraphics[scale=.35]{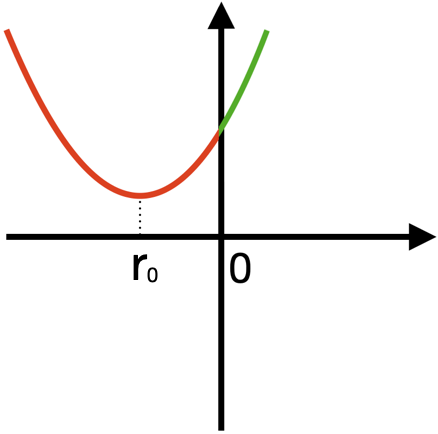}
\vskip -.1in
\caption{Lower bound for $\Delta C_{\ell_1}$. The green denotes the values $r(x)$ can take. \textbf{Left:} $r(x) \!\leq\! 0$; \textbf{Right:} $r(x) \!\geq\! 0$. In both cases, the infimum is attained at $r(x) = 0$.}
\vskip -.15in
\label{fig:estimates}
\end{figure}
\begin{restatable}{proposition}{PropositionCharacterization}
\label{prop:1}
 Assume that there exists a convex function $\Psi\colon
 \Rset_+ \to \Rset$ with $\Psi(0) = 0$  such
 that the following holds for all $r \in \Rall$ and $x\in \sX$:
$\Psi\paren*{\abs*{\eta(x) - (1 - c)} \Ind_{r(x) r^*(x) \leq 0}} \leq \Delta\sC_{\ell_1}(0, x)$.
Let $\Ic$ be defined by $\Ic = c e^{\frac{\alpha}{2}} + (1 - c) e^{-\frac{\alpha}{2}}$ and assume that $\frac{2 \beta c}{\alpha} = \Ic$.
Then, the following inequality holds for any $r \in \sR$:
\begin{equation}
     \Psi\paren*{R_{\ell_2}(r) - R_{\ell_2}^*}
     \leq  R_{\ell_1}(r) - R_{\ell_1}^* .
\end{equation}
\end{restatable}
The result shows that, instead, we only need to find a function
$\Psi$ relating the gap $\Delta \sC_{\ell_2}$ to $\Delta\sC_{\ell_1}(0, x)$, which is no
longer a quantity depending on $r(x)$.
To do so, we look to lower-bound $\Delta \sC_{\ell_1}$ over the infimum of $r(x)$. Since $\Delta \sC_{\ell_1}$ is a (strictly) convex function of $r(x)$, if we can select the parameter $\alpha$ and $\beta$ to ensure $r^*(x) \!>\! 0 \Leftrightarrow r_0(x) \!>\! 0$, where $r_0$ is the Bayes solution for the surrogate rejection loss, then this infimum occurs at $r(x) = 0$. This is illustrated in Figure~\ref{fig:estimates}. Proposition~\ref{prop:1} states that this can be arranged if $\alpha$ and $\beta$ are related by $\frac{2 \beta c}{\alpha} = \Ic$. In view of this proposition, we will adopt the assumption $\frac{2 \beta c}{\alpha} = \Ic$ and analyze $\Delta\sC_{\ell_1}(0, x)$. Note that the equivalence proven in the proof holds if and only if this equality holds.

\begin{theorem}
Let $\alpha, \beta > 0$ be such that
$\frac{2 \beta c}{\alpha} = \Ic$, where 
$\Ic = c e^{\frac{\alpha}{2}} + (1 - c) e^{-\frac{\alpha}{2}}$.
Then, the following inequality holds for any $r \in \Rall$:
\ifdim\columnwidth=\textwidth
{
\begin{align*}
    R_{\ell_2}(r) - R_{\ell_2}^* 
    \leq  \frac{2}{e^{\frac{\alpha}{2}} - e^{-\frac{\alpha}{2}}}\sqrt{\frac{(c +  \Ic)\Ic}{c} \paren*{R_{\ell_1}(r) - R_{\ell_1}^*} }  .
\end{align*} 
}
\else{
\begin{align*}
    & R_{\ell_2}(r) - R_{\ell_2}^* \\
    & \leq  \frac{2}{e^{\frac{\alpha}{2}} - e^{-\frac{\alpha}{2}}}\sqrt{\frac{(c +  \Ic)\Ic}{c} \paren*{R_{\ell_1}(r) - R_{\ell_1}^*} }  .
\end{align*} 
}
\fi
\end{theorem}
The theorem shows that if
the excess surrogate loss $\paren*{R_{\ell_1}(r) - R_{\ell_1}^*}$ is
reduced to $\e$, then the excess rejection loss
$\paren*{R_{\ell_2}(r) - R_{\ell_2}^*}$ is bounded by $O(\sqrt{\e})$. This provides a strong guarantee for the surrogate loss function considered when the condition $\frac{2 \beta c}{\alpha} = \Ic$ holds. Similar results can be derived for other family of functions $\sR$, such as that of linear functions or neural networks with one hidden-layer as in \citep*{AwasthiMaoMohriZhong2022}. 
This gives a \textit{principled} method for defining the relation between $\alpha$
and $\beta$. The value of the other parameter, say $\alpha$, can be set arbitrarily or via a hyper-parameter search.
\ignore{
We recommend a hyper-parameter search over $\alpha$ and $\beta$, but this gives a \textit{principled} method for defining the relation between these two parameters. 
}

\subsection{Visualization of surrogate loss}

In Figure~\ref{fig:loss_plt}, we plot our surrogate loss as a function of $r(x)$ on $(-0.5, +0.5)$, and arbitrarily choose $c=0.05$ and $\alpha=2$ with $\beta$ following the relationship defined in $\frac{2 \beta c}{\alpha} = \Ic$. We include the plot for both negatively-annotated points ($a = -1$) and positively-annotated points ($a = +1$). The first term is always increasing, and the second is always decreasing. 

We observe the following property: for negatively-annotated points, the minimum is attained at $r(x) < 0$, and for positively-annotated points, while it may be difficult to see, the minimum is attained at $r(x) > 0$. 
The following is a key insight from our $\sH$-consistency bound: this property holds for any $c$, and any $\alpha$ and $\beta$ satisfying $\frac{2 \beta c}{\alpha} = \Ic$, as the signs of $r^*(x)$ and $r_0(x)$ match. 
This relationship thus ensures that the minimums of both plots are in the proper regions.

\begin{figure}[t]
\centering
\includegraphics[scale=.29]{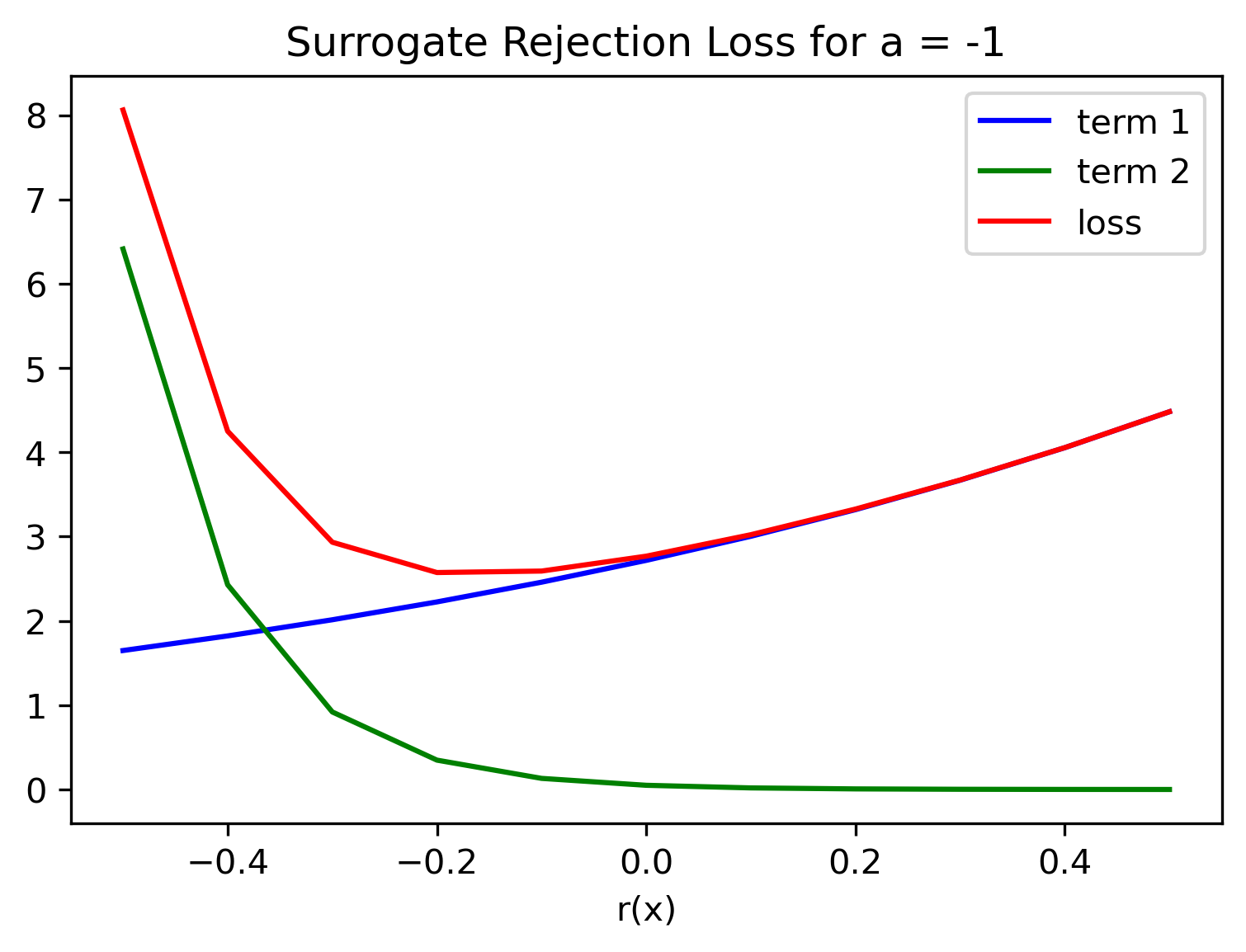}
\includegraphics[scale=.29]{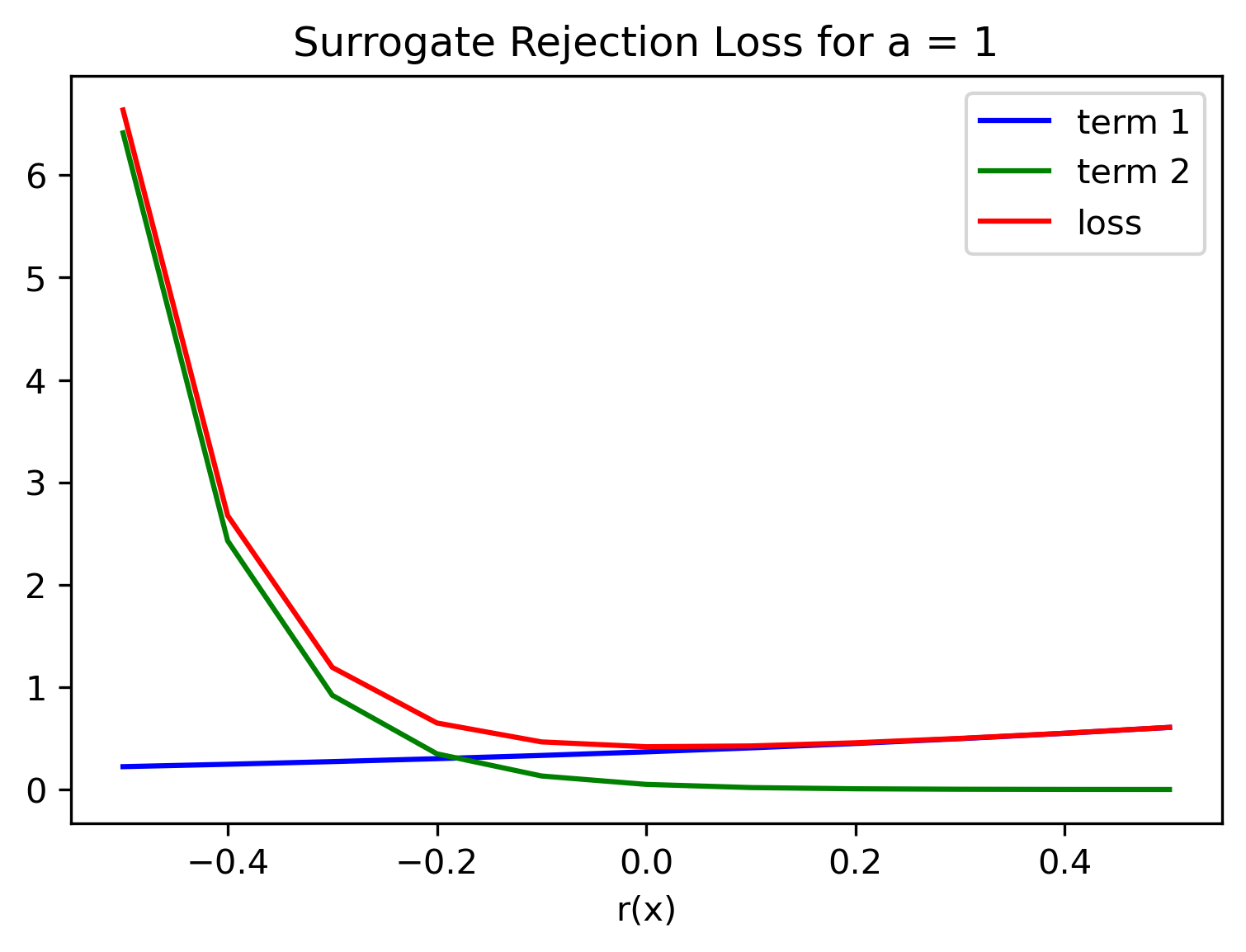}
\vskip -.1in
\caption{Surrogate rejection loss as a function of $r(x)$ when $c=0.05, \alpha=2$, and $\beta$ following $\frac{2 \beta c}{\alpha} = \Ic$, with both terms in the sum. \textbf{Left:} negatively-labelled points (-1). \textbf{Right:} positively-labelled points (+1). }
\vskip -.15in
\label{fig:loss_plt}
\end{figure}

\section{Decontextualization task}
\label{sec:decontext-task}

We chose the NLP task of decontextualization for evaluation. This is a challenging task because only a modest amount of annotated data is available and because each input typically 
admits a large number of correct labels. We give a short description of the task and our annotation procedure.

\ignore{\ec{can we find some statistics to back it up? maybe binary match accuracy among annotation pairs? also this is probably true for many text generation tasks, might worth mentioning somewhere}}

\subsection{Definition}

Decontextualization is the task of editing a sentence within a passage 
so that it can be interpretable out of context \citep{Collins2021}.
Specifically, given a sentence-context pair $(s,c)$, a sentence $s'$ is a valid 
decontextualization of $s$ if: (1) the sentence $s'$ is interpretable in the empty 
context; and (2) the truth-conditional meaning of $s'$ in the empty context is 
the same as the truth-conditional meaning of $s$ in context $c$. We refer readers to
\citep{Collins2021} for a full description. \ignore{This task is highly non-deterministic: there could be a large number of valid decontextualizations.}

\subsection{Annotation}

For our experiments, we labeled $2\mathord,000$ decontextualizations 
of a fixed MT5 XXL model \citep{MT5} ourselves, fine-tuned on the decontextualization
task. The training data for the original decontextualization task is a sample from the 
English portion of the Wikipedia corpus \citep{Collins2021}. The input is formed
by concatenating the title and subtitle of the relevant page with `[HEAD]', and then appending the relevant paragraph
after `[SEP]', see Figure~\ref{fig:title-example}.
Our 2,000 annotated decontextualizations 
are originally from a random sample of English Wikipedia. Examples that the model labelled as `impossible' 
or `unnecessary' were not considered. We observe that annotating for this task is difficult: some take several minutes to evaluate.

\begin{figure}
\centering
\examplebox{
    Goldie Taylor [HEAD] Career ; Corporate [SEP] Taylor has worked for the Sara Lee Corporation as director of global communications and public affairs. \textbf{ \underline{Goldie} Taylor has served as executive consultant to NBC News and CNN Worldwide.}}
\caption{Decontextualization labeled as `title'.}
\vskip -.15in
\label{fig:title-example}
\end{figure}

When labeling or evaluating the validity of a decontextualization,
we consider the correctness of the edits: the added information must be correct, and the deletions
cannot change the meaning of the sentence. Sometimes, however, 
this is impossible to discern without using information from the title or subtitle. We thus labeled the outputs as `yes', `no' or `title'. We present a `title' example in Figure~\ref{fig:title-example}. The decontextualization request is for the bolded sentence, and `Goldie' is then inserted. However, 
`Goldie' only appears in the title. While it is probable that `Taylor' refers to `Goldie
Taylor', we must rely on the information from the title. It is however also possible that `Taylor' refers to a family member of `Goldie Taylor' and that the paragraph is entirely unrelated to the title.

In our case, since `title' examples are likely to be factual (while unsupported by the context provided to the model), we evaluate experimentally by including them with `yes'. In other setting such as novels, `title' examples are less likely to be factual, as paragraphs deep within a chapter have little connection to their title. 

\section{Experimental evaluation}
\label{sec:experiments}

In this section, we report results for the described learning methods.

\subsection{Dataset}

We randomly split our 2,000 annotations into 1,500 train/500 validation examples and perform 4-fold cross-validation. 1,019 (50.95\%) of the annotations are `yes', 761 (38.05\%) are `title', and the remaining 220 (11.00\%) are `no.' As already mentioned, we consider the `title' examples as `yes', so we have about 89\% positive examples. The \textit{decontextualization rejection task} is constructed as follows: we concatenate the input and output of the decontextualization model with ` [OUT] ' to form the input. The target consists of just one token, `yes' or `no.'

\ignore{
For all experiments we use a T5X 1.1 XXL decontextualization model \citep{roberts2022t5x} trained on .... data. This model is a transformer architecture with an input length of X and an output length of Y in addition to the yes/no classification task output unit. We limit training to the 'yes' token.  }

\subsection{Methods}

\noindent\textbf{Maxprob}: We use the highest scoring output sequence as determined by beam search as the output sequence and use the sum of the logits for each token of that sequence as the score. The best threshold for some precision level is determined on the training data, and then evaluated on the validation data. 

\noindent\textbf{Cross-entropy loss}: We further fine-tune a T5X 1.1 XXL decontextualization model \citep{roberts2022t5x}, limited to one output token, on the decontextualization rejection task, and use as the score the value of the  $\logits_{yes}$. The standard cross-entropy loss function is used, and a threshold is similarly fitted on half of the validation data and evaluated on the other half. We perform a hyper-parameter search over $\{1e-4, 1e-3, 1e-2\}$ for the learning rate, and $\{0, 0.05, \ldots, 0.2\}$ for the dropout rate.

\noindent\textbf{Surrogate loss}:
In our formulation, we have a rejector $r\colon \sX \to \Rset$, different from the sequence output of the T5X \ignore{1.1 XXL }model described. Thus, we use $\softmax ( \logits )_{yes} \!-\! 0.5$ as the value of $r(x)$. We further fine-tune the same T5X 1.1 XXL decontextualization model, but with our surrogate loss function. For the two most extreme rejection levels presented, we maintain the value of $c$ but fit a threshold slightly different from 0. We set $\alpha$ to 4, and do not perform a hyper-parameter search.

\noindent\textbf{Theoretical limit}: The theoretical limit of coverage can be defined as $b / p$, where $b$ is the fraction of positively labeled points and $p$ is the desired precision level. The precision is obtained exactly, and standard deviations for coverage are a result of the slightly different class imbalances in the cross-validation splits. 

\begin{figure}[t]
\centering
\includegraphics[scale=0.33]{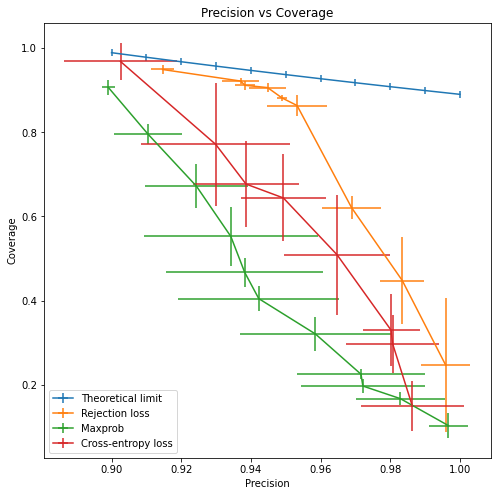}
\vskip -.1in
\caption{Precision vs Coverage on Decontextualization. Standard deviations are from the four cross-validation splits.}
\vskip -.15in
\label{fig:pc}
\end{figure}

\subsection{Discussion}

We observe that our rejection loss significantly outperforms the baselines considered, 
and for the lower half of the precision levels, \textit{closely follows the theoretical
limit}. We provide, at $c=0.07$ for example, about a halving of the error rate
(11.00\% to 5.1\%) while maintaining a broad 90.6\% coverage. The theoretical limit for 5\% 
error is only slightly higher at 93.6\% coverage, and the closest baseline, the cross-entropy loss, 
only provides  $\sim\!\!64.4\%$ coverage at $\sim\!\!5.1\%$ error. 

Another important observation is the stability of this result. Not only does the surrogate loss
perform much better on average, the standard deviations are also significantly smaller. This gives a better sense of the performance of an individual rejector model trained with the surrogate loss function.

\ignore{
Why do we believe the surrogate loss performed well on decontextualization? 
There are many correct outputs, and thus a rich rejection space we are \textit{properly} minimizing over.
This is in addition a difficult rejection task: even when annotating, we found a large fraction of examples took several minutes to evaluate. 
For easier tasks, such as collapsing `title' to `yes', where the evidence for a rewrite is often immediately next to it, we observe that the trained methods all have AUC $> 0.90$.
Thus, while we observe that our surrogate loss can be used for other tasks, and only requires binary labels indicating correctness, we caution the reader against expecting improvement when the task is easy.  
}

\section{Conclusion}

\ignore{We introduced a new method for classification with a reject option while maintaining a fixed predictor, that achieves remarkable results on the decontextualization task and benefits from strong theoretical guarantees. 
It simply requires labeling a moderate amount of data and then training a binary classifier with a modified loss function.

We observe that our algorithm can be used in other settings where a binary label indicating correctness of a prediction is available. We encourage the use of our algorithm in difficult rejection tasks with a large output space and a large number of correct labels. In particular, our algorithm can be used for abstention with pre-trained large language models, 
in a variety of contexts.}

We presented a theoretically-justified approach to classification with a reject option for applications where the predictor remains fixed. Our main contributions include the following: 
(1)  a new formulation of the problem of learning a rejector with fixed predictor (for cases where there may be many correct labels);
(2) introduction of a new surrogate loss function for our scenario, with the proof of a strong $H$-consistency bound guarantee;
(3)  definition of the notion of correctness for decontextualization, and its use to provide a dataset of $2\mathord,000$ manually-labeled decontextualizations;
(4) experimental results demonstrating a $10$-$25\%$ improvement over baselines in coverage at various precision levels on decontextualization.

We observe that our algorithm can be used in other settings where a binary label indicating correctness of a prediction is available. We encourage the use of our algorithm in difficult rejection tasks with a large output space and a large number of correct labels. In particular, our algorithm can be used for abstention with large language models, 
in a variety of contexts such as text generation.

\subsection*{Acknowledgments}
We warmly thank Anqi Mao and Yutao Zhong for discussions related to an earlier version of this manuscript and clarifications related to $\sH$-consistency.

\newpage
\bibliography{dcont}
\bibliographystyle{abbrvnat}

\newpage
\appendix
\onecolumn

\ignore{
\renewcommand{\contentsname}{Contents of Appendix}
\tableofcontents
\addtocontents{toc}{\protect\setcounter{tocdepth}{3}} 
\clearpage
}

\label{app:consistency_proof}
\section{$\sH$-Consistency bound proof}

\subsection{Calibration gap for rejection loss}

The following gives the expression of the calibration gap $\Delta \sC_{\ell_2}$.

\CalibrationGapRejectionLoss*

\begin{proof}
For any $r \in \Rall$ and $x \in \sX$, we can write
\begin{align*}
    \sC_{\ell_2}(r, x) & = \eta(x) \ell_2(r, x, +1) + [1 - \eta(x)] \ell_2(r, x, -1) \\
    & = \eta(x) \left[\Ind_{+1=-1}\Ind_{r(x) > 0} + c \Ind_{r(x) \leq 0} \right] + [1 - \eta(x)] \left[\Ind_{-1=-1}\Ind_{r(x) > 0} + c \Ind_{r(x) \leq 0}\right] \\
    \ignore{
    & = \eta(x) \left[c \Ind_{r(x) \leq 0} \right] + [1 - \eta(x)] \left[\Ind_{r(x) > 0} + c \Ind_{r(x) \leq 0}\right] \\
    & = \eta(x) \left[c \Ind_{r(x) \leq 0} \right] + [1 - \eta(x)] \left[\Ind_{r(x) > 0} \right] + [1 - \eta(x)] \left[ c \Ind_{r(x) \leq 0}\right] \\}
    & = c \Ind_{r(x) \leq 0}  + [1 - \eta(x)] \Ind_{r(x) > 0}.
\end{align*}
For the optimal $\sC_{\ell_2}^*$, we would always pick the lower of $c$ or $1 - \eta(x)$, which gives:
$\sC_{\ell_2}^*(x) = \min \{c, 1 - \eta(x)\}$.
The corresponding Bayes solution $r^*$ can be defined by $r^*(x) = \eta(x) - (1 - c)$.
Thus, the calibration gap is given by
\begin{align*}
    \Delta \sC_{\ell_2}(r, x) = c \Ind_{r(x) \leq 0}  + [1 - \eta(x)] \Ind_{r(x) > 0} - \min \curl*{c, 1 - \eta(x)}.
\end{align*}
If $r(x)$ correctly chooses the lower of the two, we have $r(x)r^*(x) > 0$ and then $\Delta \sC_{\ell_2} = 0$. Otherwise, we have
\begin{align*}
    \Delta \sC_{\ell_2}(r, x) 
    & =  
    \begin{cases} 
      c - (1-\eta(x)) & \text{if } r(x) \leq 0 \\
      (1-\eta(x)) - c & \text{otherwise} \\
   \end{cases}.
\end{align*}
Thus, for all $x \in \sX$, we have
   $\Delta \sC_{\ell_2}(r, x) = \abs*{\eta(x) - (1 - c)} \Ind_{r(x) r^*(x) \leq 0}$.
This completes the proof.
\end{proof}

\ignore{
Note, LHS incorporates $\e$ now - I believe this is correct.
}

\subsection{Calibration gap for surrogate loss}

Here, we analyze the calibration gap for the surrogate loss.

\CalibrationGapSurrogateLoss*

\begin{proof}
By definition, the calibration function for $\ell_1$ can be expressed for all $x \in \sX$ by
\begin{align*}
    \sC_{\ell_1}(r, x) & = \eta(x) \ell_1(r, x, +1) + [1 - \eta(x)] \ell_1(r, x, -1) \\
    & = \eta(x) \left[e^{\frac{\alpha}{2}[r(x) - 1]} + ce^{-\beta r(x)} \right] 
    + [1 - \eta(x)] \left[e^{\frac{\alpha}{2}[r(x) + 1]} + ce^{-\beta r(x)}\right] \\
    & = e^{\frac{\alpha}{2} r(x)} I_\eta(x) + c e^{-\beta r(x)}.
\end{align*}
Since the exponential function is convex, $\Delta \sC_{\ell_1} (r, x)$ is a convex function of $r(x)$. Thus,
for $r \in \sR_{all}$, we obtain the minimum $r_0(x)$ by differentiating with respect to $r(x)$ and setting to 0:
\begin{align*}
    & \frac{\alpha}{2} e^{\frac{\alpha}{2} r(x)} I_\eta(x) - \beta ce^{-\beta r(x)} = 0 
    \Leftrightarrow r_0(x) =  \log  \bracket*{ \paren*{\frac{2 \beta c}{\alpha I_\eta(x)} }^{\frac{2}{2 \beta + \alpha}} }.
\end{align*}
Plugging in this expression in $\sC$ gives 
the corresponding minimal calibration $\sC^*_{\ell_1}(x)$:
$ \sC^*_{\ell_1}(x)
    = \left[ \left( \frac{2 \beta c}{\alpha} \right)^{\gamma} \right]  I_\eta(x)^{1-\gamma} \left(\frac{1}{1-\gamma}\right)$.
This completes the proof.
\end{proof}

\subsection{$\sH$-consistency bound}

In this section, we prove our main result. The following will provide a key tool to derive our result.

\ignore{
\begin{figure}[t]
\centering
\includegraphics[scale=.35]{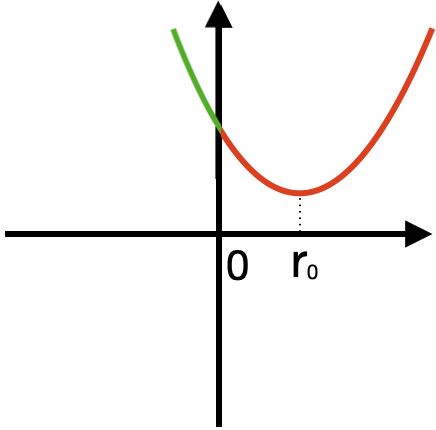}
\hspace{1cm}
\includegraphics[scale=.35]{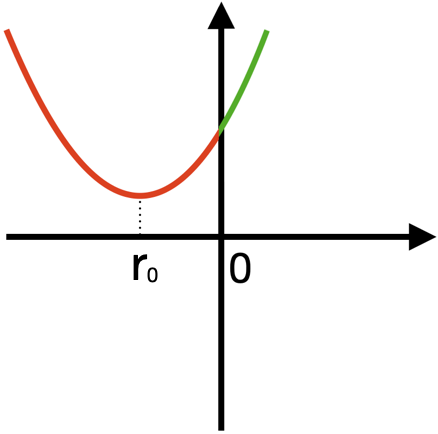}
\vskip -.1in
\caption{Lower-bound for $\Delta C_{\ell_1}$. The green denotes the values $r(x)$ can have. \textbf{Left:} $r(x) < 0$. \textbf{Right:} $r(x) > 0$. In both cases, the infimum is attained at r(x) = 0.}
\end{figure}
}

\PropositionCharacterization*
\begin{proof}
We will show that $\inf_{r(x) r^*(x) \leq 0} \Delta \sC_{\ell_1}(r, x) = \Delta \sC_{\ell_1}(0, x)$. The result then follows by Theorem~\ref{th:excess_bounds_Psi} and Lemma~\ref{lemma:1}. Since we have $\frac{2 \beta c}{\alpha} = \Ic$, the following equivalence holds:
\begin{align*}
r_0(x) > 0 
\Leftrightarrow \frac{2 \beta c}{\alpha I_\eta(x)} > 1
\Leftrightarrow I_\eta(x) < \Ic
\Leftrightarrow \eta(x) > \tfrac{e^{\frac{\alpha}{2}} - \Ic}{e^{\frac{\alpha}{2}} - e^{-\frac{\alpha}{2}}}
\Leftrightarrow \eta(x) > \tfrac{(1 - c)e^{\frac{\alpha}{2}} - (1 - c)e^{-\frac{\alpha}{2}} }{e^{\frac{\alpha}{2}} - e^{-\frac{\alpha}{2}}}
\Leftrightarrow r^*(x) > 0.
\end{align*}
This implies $\inf_{r(x) r^*(x) \leq 0} \sC_{\ell_1}(r, x)
= \inf_{r(x) r_0(x) \leq 0} \sC_{\ell_1}(r, x)$.
Now, since $r_0(x)$ is the unique minimizer of the strictly convex function
$\sC_{\ell_1}(r, x)$ of $r(x)$, then,
as a function of $r(x)$, $\sC_{\ell_1}(r, x)$ is decreasing from $-\infty$ to 
$r_0(x)$ and increasing from there to $+\infty$.
Thus, if $r_0(x) > 0$,
the infimum of $\sC_{\ell_1}(r, x)$ 
over $r(x) \leq 0$ is reached for
$r(x) = 0$. Similarly, if 
$r_0(x) < 0$, the infimum of $\sC_{\ell_1}(r, x)$ 
over $r(x) \geq 0$ is reached for
$r(x) = 0$. This shows that
$\inf_{r(x) r_0(x) \leq 0} \sC_{\ell_1}(r, x)
= \sC_{\ell_1}(0, x)$,
and completes the proof.
\end{proof}
\ignore{
In view of this proposition, we will adopt the assumption 
$\frac{2 \beta c}{\alpha} = \Ic$ and analyze 
$\Delta\sC_{\ell_1}(0, x)$. Note that the equivalence proven in the proof holds if and only if this equality holds.
}
The proof of our main result makes use of the following identity, which is a refinement of Bernoulli's inequality. The result could be of independent interest in other contexts, we give a concise proof below.

\begin{lemma}[Bernoulli-type inequality]
\label{lemma:identity}
The following identity holds for all $x, r \in (0, 1)$,
\[
(1 + x)^r \leq 1 + rx + \frac{r(r - 1)x^2}{4}.
\]
\end{lemma}
\begin{proof} Let $f_r(x) = (1 + x)^r - \paren*{1 + rx + \frac{r(r - 1)x^2}{4}}$. We will show that $f_r(x) \leq 0$ for all $x, r \in (0, 1)$. We have $f_r'(x) = r(1 + x)^{r-1} - \paren*{r + \frac{r(r - 1)x}{2}}$, and  $f_r'(0) = 0$. To see that $f_r'(1) \leq 0$, observe $r2^{r-1} - \paren*{r + \frac{r(r - 1)}{2}} \leq 0 \Leftrightarrow 2^{r-1} - \frac{(r - 1)}{2} \leq 1$. The left-hand side of the last inequality is a convex function of $r$, and equal to 1 when $r=0$ or $r=1$. Thus, the left-hand side is less than or equal 1 for $r \in (0, 1)$, giving $f_r'(1) \leq 0$.  Since $f_r'(x)$ is a convex function of $x$, with $f_r'(0) \leq 0$ and $f_r'(1) \leq 0$, then $f_r'(x) \leq 0$ for all $x \in (0, 1)$, which shows $f_r$ is decreasing. Then, since $f_r(0) = 0$, $f_r(x) \leq 0$ for all $x,r \in (0, 1)$, which completes the proof.
\end{proof}
The following is our main result; it relates
the surrogate excess error to that of the rejection loss.

\begin{theorem}
Let $\alpha, \beta > 0$ be such that
$\frac{2 \beta c}{\alpha} = \Ic$, where 
$\Ic = c e^{\frac{\alpha}{2}} + (1 - c) e^{-\frac{\alpha}{2}}$.
Then, the following inequality holds for any $r \in \Rall$:
\[
     R_{\ell_2}(r) - R_{\ell_2}^*
     \leq  \frac{2}{e^{\frac{\alpha}{2}} - e^{-\frac{\alpha}{2}}}\sqrt{\frac{(c +  \Ic)\Ic}{c} \paren*{R_{\ell_1}(r) - R_{\ell_1}^*} }  .
\]     
\end{theorem}

\begin{proof}
Using the expression of $\Delta \sC_{\ell_1}$ given by Lemma~\ref{lemma:2}, we can write
\begin{align*}
    \Delta\sC_{\ell_1}(0, x)
    = I_\eta(x) + c 
    - \frac{1}{1 - \gamma} \paren*{ \frac{2 \beta c}{\alpha} }^{\gamma}  I_\eta(x)^{1 - \gamma} 
    & = I_\eta(x) + c 
    - \paren*{\Ic + c}  \paren*{\frac{I_\eta(x)}{\Ic}}^{1 - \gamma}.
\end{align*}
We can express this formula in terms of $u(x) = \eta(x) - (1 - c)$, using $I_\eta(x) = J_u(x) + \Ic$, with $J_u(x) =  \bracket*{e^{\frac{\alpha}{2}} - e^{-\frac{\alpha}{2}}} u(x)$:
\begin{align*}
    \Delta\sC_{\ell_1}(0, x)
    = J_u(x) + \Ic + c- \paren*{\Ic + c} \bracket*{1 + \frac{J_u(x)}{\Ic}}^{1 - \gamma}
    \mspace{-20mu}\geq \frac{\Ic}{c + \Ic} \frac{c}{c + \Ic} \frac{c + \Ic}{4} \bracket*{\frac{J_u(x)}{\Ic}}^2 
    = \frac{1}{4} \frac{c \Ic}{c + \Ic} 
    \bracket*{\frac{J_u(x)}{\Ic}}^2.
\end{align*}
where we used Lemma~\ref{lemma:identity}.
The function $\Psi(u)$ defined by this expression verifies the condition of Proposition~\ref{prop:1} and therefore we have
     $\Psi\paren*{R_{\ell_2}(h) - R_{\ell_2}^*}
     \leq  R_{\ell_1}(h) - R_{\ell_1}^*$.
An explicit upper-bound on $R_{\ell_2}(h) - R_{\ell_2}^*$ can be written in terms of $\Psi^{-1}$:
     $R_{\ell_2}(h) - R_{\ell_2}^*
     \leq  \Psi^{-1}\paren*{R_{\ell_1}(h) - R_{\ell_1}^*}$.
To derive the expression of $\Psi^{-1}$, we write $z = \Psi(u)$, that is:
\begin{equation*}
    \ 4 \frac{c + \Ic}{c \Ic} z = \bracket*{\frac{u(x)}{\Ic}}^2 
    \bracket*{e^{\frac{\alpha}{2}} - e^{-\frac{\alpha}{2}}}^2
    \Leftrightarrow \ |u| = \frac{2}{e^{\frac{\alpha}{2}} - e^{-\frac{\alpha}{2}}}\sqrt{\frac{(c +  \Ic)\Ic}{c}  z }.
\end{equation*}
Thus, we have for all $r \in \Rall$,
    $R_{\ell_2}(r) - R_{\ell_2}^*
     \leq  \frac{2}{e^{\frac{\alpha}{2}} - e^{-\frac{\alpha}{2}}}\sqrt{\frac{(c +  \Ic)\Ic}{c} \paren*{R_{\ell_1}(r) - R_{\ell_1}^*} }$  .
%
\end{proof}

\end{document}